\documentclass[onecolumn,journal]{IEEEtran}

\usepackage{arxiv}

\usepackage[T1]{fontenc}

\ifCLASSINFOpdf
\else
\fi
\usepackage{microtype}
\usepackage{graphicx}
\usepackage{subfigure}
\usepackage{booktabs} 
\usepackage{array,multirow}
\usepackage{enumitem}

\usepackage{algorithm}
\usepackage[noend]{algorithmic}

\newcounter{ALC@tempcntr}
\newcommand{\NEWCOMMENT}[1]{%
    \setcounter{ALC@tempcntr}{\arabic{ALC@rem}}
    \setcounter{ALC@rem}{1}
    \item {\footnotesize \it  //  #1}  
    \setcounter{ALC@rem}{\arabic{ALC@tempcntr}}
}%

\newcommand{\NEWSECTION}[1]{%
    \setcounter{ALC@tempcntr}{\arabic{ALC@rem}}
    \setcounter{ALC@rem}{1}
    \item {  \bf  #1 }  
    \setcounter{ALC@rem}{\arabic{ALC@tempcntr}}
}%

\usepackage{setspace,cite}
\usepackage{amsmath,amsthm,amsfonts}
\usepackage{bm}
\usepackage{amssymb}

\allowdisplaybreaks[1] 

\newcommand\inner[2]{\langle #1, #2 \rangle}

\DeclareMathOperator{\bw}{\mathbf{w}}

\makeatletter
\newcommand*{\rom}[1]{\expandafter\@slowromancap\romannumeral #1@}
\makeatother

\newtheorem{theorem}{Theorem}
\newtheorem*{proof*}{proof}
\newtheorem{lemma}{Lemma}

\newtheorem{assumption}{Assumption}

\begin{document}

\title{Sustainable Federated Learning}

\author{\vspace{0.2in}\onehalfspacing\IEEEauthorblockN{Ba\c{s}ak~G{\"u}ler\IEEEauthorrefmark{1} \quad  Aylin Yener\IEEEauthorrefmark{7} 
} \vspace{0.1cm}\\
\IEEEauthorblockA{\IEEEauthorrefmark{1}University of California, Riverside\\
Department of Electrical and Computer Engineering\\
Riverside, CA 92521 \\
{\em bguler@ece.ucr.edu}}
\vspace{0.5cm}\\
\IEEEauthorblockA{\IEEEauthorrefmark{7}The Ohio State University\\ Department of Electrical and Computer Engineering\\
Columbus, OH 43210 \\
{\em yener@ece.osu.edu}} 
}


\maketitle
\begin{abstract}
Potential environmental impact of machine learning by large-scale wireless networks is a major challenge for the sustainability of future smart ecosystems. In this paper, we introduce sustainable machine learning in federated learning settings, using rechargeable devices that can collect energy from the ambient environment. We propose a practical federated learning framework that leverages intermittent energy arrivals for training, with provable convergence guarantees. Our framework can be applied to a wide range of machine learning settings in networked environments, including distributed and federated learning in wireless and edge networks. Our experiments demonstrate that the proposed framework can provide significant performance improvement over the benchmark energy-agnostic federated learning settings.
\end{abstract}



\section{Introduction}
The environmental impact of machine learning matters. Modern machine learning systems consume massive amounts of energy.  
In fact, the computational resources needed to train a state-of-the-art deep learning model has increased by 300000x between 2012-2018 \cite{schwartz2020green}. Today, it is estimated that training a single deep learning model can generate as much $CO_2$ as the total lifetime of five cars \cite{strubell2019energy}. 
This impact will worsen with the emergence of machine learning in distributed and federated learning settings, where billions of devices are expected to train machine learning models on a regular basis. 
In this paper, we provide a first study for sustainable machine learning in the federated learning setting, through the use of compute devices that can generate energy from renewable sources in the ambient environment, such as solar, kinetic, ambient light, or ambient RF energy \cite{radousky2012energy, li2018enabling, chang2018insight}.

Federated learning is a communication-efficient and privacy-preserving distributed learning framework for training machine learning models over large volumes of data created and stored locally at millions of remote clients, such as the data generated at mobile or edge devices \cite{pmlr-v54-mcmahan17a, mcmahan2017communication, bonawitz2019towards, kairouz2019advances, li2020federated}. 
This is an iterative training process that is coordinated by a central server. The server maintains a global model and sends its current state to the clients at the beginning of each training iteration.  
During training, the data collected by the individual devices never leaves the device, instead, devices locally update the global model using their local dataset, creating a local model. The local models are then sent  to the central server, who then aggregates the local models to update the global model. 
It has received significant attention in the recent years and has found a variety of applications from keyboard query recommendations to healthcare, from electrical load forecasting to traffic flow prediction \cite{yang2018applied, konevcny2016federated, bonawitz2019federated, li2019fair, bhagoji2018analyzing, geyer2017differentially, taik2020electrical, saputra2019energy, liu2020privacy, xu2020federated, yang2019federated, chen2020fedhealth, yu2020learning}. 

Recent works have considered energy efficient training strategies for federated learning  \cite{zeng2020energy, wang2019adaptive}. In these works, the primary focus is on either minimizing the total energy cost of training \cite{zeng2020energy}, or minimizing the training loss within a given total energy budget \cite{wang2019adaptive}, where all of the energy is available at the beginning of training.  
In contrast, our focus is on federated learning when devices generate energy through an intermittent and non-homogeneous renewal process. Our goal is to build a scalable and practical federated learning framework with provable convergence guarantees, for networks with intermittent energy arrivals.  

Prior work has investigated client selection in the context of  federated learning primarily for improving the convergence rate or reducing the communication overhead of training \cite{LiHYWZ20, DBLP:conf/nips/AgarwalSYKM18, chen2020optimal, goetz2019active, cho2020client}.
Such algorithms are designed with the underlying assumption that all clients are available to participate in training if selected. Then, the goal is to either sample a small number of clients uniformly at random to minimize the communication overhead per iteration, or to select the clients that maximize the convergence rate of training. 
In contrast, in our setting, whether or not a client can participate in training is determined by an underlying energy arrival process, which is intermittent and non-homogeneous across the clients. 
Several works have considered federated learning when clients may dropout from the system during training, however, in these setups, the main assumption is that the client dropouts occur uniformly at random, which does not bias the training process  \cite{mcmahan2017communication, bonawitz2017practical}.

In this work, we consider a federated learning scenario with $N$ clients and a server. Client $i$ holds a local dataset $\mathcal{D}_i$. Clients wish to jointly train a machine learning model over the datasets $\mathcal{D}_1, \ldots, \mathcal{D}_N$. Training is   coordinated by the central server, who maintains a global model that is updated locally by the clients through an iterative process. 
More specially, at each training iteration, the server sends the current state of the global model to the clients. Then, clients locally update the global model, through multiple stochastic gradient descent (SGD) iterations over their local dataset, and send their local updates to the server. 
Finally, the server updates the global model by aggregating the local updates received from the clients, and sends the updated global model back to the clients, to be used in the next iteration. 

Unlike the conventional federated learning setting, we assume that clients only have intermittent energy availability, and can participate in the training process only when they have energy available. The energy generation process is not uniform across the devices, that is, some clients may have more frequent energy arrivals than others. One potential approach in this setting is to let each client to participate in training as soon as they generate enough energy to do so. However, as we demonstrate in our experiments, in this setting, conventional federated learning strategies may bias the global model towards clients with more frequent energy arrivals, causing a performance loss in the resulting accuracy. Another approach is to wait until all clients generate enough energy to participate in training before each iteration of the conventional federated learning scheme. 
Doing so, however, would require waiting for the clients with the slowest energy generation, therefore, even though the training is unbiased, the convergence rate can be very slow to reach a desired performance level.  

We propose a simple federated learning framework with provable convergence guarantees, for networks in which devices generate energy through an intermittent renewable energy source. The proposed framework consists of three main components, client scheduling, local training at the clients, and model update at the server. Client scheduling is performed at the client level, in other words, each client decides whether or not to participate at any given training iteration based solely on the local estimation of the energy arrival process. The client scheduling process requires no coordination between the clients, and is scalable to large networks. 
During the local training phase, clients who choose to participate at the current training iteration update the global model using their local datasets, and then send their local updates to the server. Upon receiving the local updates, the server updates the global model for the next iteration. 

In our experiments, we compare the performance of the proposed framework with benchmark federated learning settings that are agnostic to the energy-arrival process of the clients. The first benchmark is the federated learning setting in which clients participate in training as soon as they generate enough energy, and then wait for the next energy arrival. The second benchmark is the setting in which the server waits for all clients to have enough energy to participate in training before initiating a single training iteration. 
We show that the proposed framework can significantly outperform both benchmarks in terms of the test accuracy. 

This paper is a first study of sustainable federated learning, and we hope our work to open up new research directions in building sustainable federated and distributed learning schemes for large-scale networks, where millions of devices jointly train machine learning models over large volumes of data. Some of these research directions include, formalizing the fundamental performance limits of distributed training under stochastic and unknown energy arrival processes, model quantization and compression techniques that can adapt to the resource and energy arrival patterns, and  characterizing the relationship between the energy renewal processes and training performance.  

\section{Problem Formulation}

\subsection{Federated Learning Setup}\label{federated}
We consider a distributed training setting with $N$ clients. Client $i$ has a local dataset $\mathcal{D}_i$ with $D_i$ data points. The total number of data points across all clients is $D = \sum_{i\in[N]} D_i$. 
The clients are connected through a central server that coordinates the training. 

The goal is to train a model $\mathbf{w}$ to minimize a global loss function
\begin{equation}
F(\mathbf{w}) = \frac{1}{D} \sum_{i=1}^{N} \sum_{j=1}^{D_i} l(\mathbf{w}, \mathbf{x}_{ij}) \label{eq:global1}
\end{equation} 
where $l(\mathbf{w},  \mathbf{x}_{ij})$ represents the loss of a single data point $\mathbf{x}_{ij}$ in the local dataset $\mathcal{D}_i$ of client $i$. 

By defining a local loss function
\begin{equation}\label{localoss}
F_i(\mathbf{w}) = \frac{1}{D_i} \sum_{j=1}^{D_i} l(\mathbf{w}, \mathbf{x}_{ij})
\end{equation}
for client $i$, the global loss function in \eqref{eq:global1} can be written as
\begin{equation}
F(\mathbf{w}) = \sum_{i=1}^{N} p_i F_i(\mathbf{w})
\end{equation}
where $p_i := \frac{D_i}{D}$ and thus,
\begin{equation}\label{eq:sump}
\sum_{i\in[N]} p_i = 1
\end{equation}

We next provide the details of training in the conventional federated learning setting \cite{mcmahan2017communication}. In this setting, the server maintains a global model that is updated locally by the clients. The local updates are then aggregated at the server to update the global model.  As such, the training process consists of local and global update iterations.

Each iteration (local or global) is represented by a discrete time instant $t\in\{0, 1, 2, \ldots\}$. It is assumed that a global update occurs at every $T$ time instants, where $T$ is the number of local training iterations that take place between two global updates. Without loss of generality, we assume that a global update occurs when $t\mod T = 0$, and let 
\begin{equation}
\mathcal{T} = \{t: t\mod T  = 0\}
\end{equation}
denote the set of time instances at which a global update occurs, which we also refer to as {\it synchronization steps}. 


At the beginning of each global round $t\in\mathcal{T}$, the server sends the current state of the global model to the clients, which is denoted by a vector  $\mathbf{w}^{(t)}\in \mathbb{R}^d$ of dimension $d$, where $d$ is the model size. 
Then, all or a fraction of the clients update the global model $\mathbf{w}^{(t)}$ through $T$ local training iterations, using their local datasets.  
The set of clients that participate in training at a given iteration depends on the specifics of the client scheduling algorithm, which could range from all clients to a small fraction of clients. We let $\mathcal{S}_t$ denote the set of participating clients at iteration $t$. 

Local training at the clients is performed through stochastic gradient descent (SGD), in which the model parameters are updated iteratively in the negative direction of the gradient evaluated over a random sample (or a minibatch) from the local dataset. 

To present the details of the local training process, we consider a synchronization step $t\in\mathcal{T}$, at which the server sends the current estimation of the global model $\bw^{(t)}$ to the clients. We also let $\mathbf{w}_i^{(t)}$ denote the {\it local} estimation of the model parameters at client $i$ at time $t$. Accordingly,  we will call $\mathbf{w}_i^{(t)}$ the {\it local model} of client $i$ at time $t$. 

When $t\in\mathcal{T}$, client $i\in[N]$ sets its local model as,
\begin{equation}
\mathbf{w}_i^{(t)}\leftarrow \mathbf{w}^{(t)}. 
\end{equation}
In other words, at each synchronization step, clients synchronize their local models with the current state of the global model.
Then, client $i\in\mathcal{S}_t$ updates their local model through $T$ SGD iterations, 
\begin{equation}\label{localupdate}
\mathbf{w}_i^{(t+j+1)} = \mathbf{w}_i^{(t+j)} - \eta_{t+j}  \nabla F_i(\mathbf{w}_i^{(t+j)}, \xi_{i}^{(t+j)})
\end{equation}
for $j\in\{0, \ldots, T-1\}$, where $\mathbf{w}_i^{(t+0)} = \mathbf{w}^{(t)}$, $\eta_{t+j}$ is the learning rate (step size), and $\nabla F_i(\mathbf{w}_i^{(t+j)}, \xi_{i}^{(t+j)})$ denotes the stochastic gradient of client $i$, 
with $\xi_{i}^{(t)}$ representing a uniformly random sample (or a minibatch) from $\mathcal{D}_i$. The stochastic gradient is an unbiased estimator of the true gradient of client $i$, 
\begin{equation}
\mathbb{E} [\nabla F_i(\mathbf{w}_i^{(t+j)}, \xi_{i}^{(t+j)})] 
= \nabla  F_i(\mathbf{w}_i^{(t+j)}) 
\end{equation}
where $\nabla  F_i(\mathbf{w}^{(t+j)})$ is the true gradient of client $i$, i.e., the gradient of the local loss function \eqref{localoss} evaluated at $\mathbf{w}_i^{(t+j)}$. 

At the end of $T$ local SGD operations, clients $i\in \mathcal{S}_{t+T-1}$ send their local updates  $\mathbf{w}_i^{(t+T)}$ from \eqref{localupdate} to the server. For all clients not participating in the current global round, i.e., for clients $i\notin\mathcal{S}_{t+T-1}$, it is assumed that   $\mathbf{w}_i^{(t+T)} = \mathbf{w}^{(t)}$.

Finally, the server updates the global model, 
\begin{equation}\label{eq:aggregation}
\mathbf{w}^{(t+T)} = \sum_{i\in[N]} p_i \mathbf{w}_i^{(t+T)}. 
\end{equation} 
by aggregating the local models of the clients. After updating the global model, the server sends the updated global model $\mathbf{w}^{(t+T)}$ to the clients for the next iteration. 


We use the term \emph{global round} to refer to the block of $T$ time instances between two consecutive syncronization steps (global updates). In other words, global round $t$ corresponds to the block of time instances $t\in\{t, \ldots, t+T-1\}$. 

We also note that our focus is on the conventional synchronous federated learning setup, in which all clients participating at a given global round perform the same number of local training iterations and the global model is updated only at specified time instances  $t\in\mathcal{T}$. Asynchronous learning scenarios in which clients can perform varying number of local iterations and communicate their local models with the server at arbitrary time instances are interesting future directions, but are beyond our scope.


\subsection{Energy Profile of the Clients}\label{sec:energy}

In this work, we consider devices that are powered by the small quantities of energy generated from the ambient environment, such as solar, kinetic, ambient light or RF energy \cite{ tan2010energy, parks2013wireless, kim2014ambient, gorlatova2015movers}. 
Clients can participate in training  only if they have available energy to do so. 


It is assumed that it takes $E_i$ global rounds for device $i$ to generate enough energy to participate in one global round of training, which includes the energy cost of computing the $T$ local updates from \eqref{localupdate} and communicating it with the server. 
We call $E_i$ the energy renewal cycle of client $i$. 

As we demonstrate in our experiments, in this setting, i.e., when clients have intermittent energy arrivals, the conventional federated learning setup from Section~\ref{federated} might bias the model towards clients with more frequent energy availability. This calls for an energy-aware client scheduling and training strategy which we study in this paper.

{\bf Main Problem.} Given the above energy arrival and training setup, the main problem we study in this work is, {\it ``How to design a scalable federated learning framework for devices with intermittent energy availability?''}. In the sequel, we provide a practical federated learning framework that takes into account the energy limitations of the clients during training, while ensuring theoretical convergence guarantees.

\section{Federated Learning  with Intermittent Resource Arrivals} 


We now introduce a practical federated learning framework for networks in which devices have intermittent energy availability. The overall procedure of our framework is provided in Algorithm~\ref{Alg}. Our framework consists of three main components, client scheduling, local training at the clients, and global model update at the server. 

\begin{algorithm}[!t]
  \caption{Federated Learning with Intermittent Resource Arrivals}\label{Alg} 
  \begin{algorithmic}[1]
    \INPUT{Number of devices $N$, local dataset $\mathcal{D}_i$ and energy renewal cycle $E_i$ of device $i\in[N]$, number of local training iterations $T$ at each global round, total number of training iterations $K$ where  $\frac{K}{T \times E_i} \in \mathbb{Z}^+$ for $i\in[N]$. 
    } \
    \OUTPUT{Global model  $\bw^{(K)}$.} \ 
    \vspace{0.2cm}
    \NEWSECTION{Initialization:}
    \vspace{0.1cm}
    \FOR{client $i=1,\ldots,N$}
    \STATE Initialize $I_i^t \leftarrow  0$ for $t\in[K]$ and $i\in[N]$. \NEWCOMMENT{Indicates whether client $i$ participates at iteration $t$.}
    \ENDFOR
    
     \vspace{0.2cm}
    \NEWSECTION{Training:}
    \vspace{0.1cm}

    \FOR{iteration $t=1,\ldots,K$}
    \vspace{0.1cm}

     \vspace{0.1cm}
    \NEWSECTION{Clients:}
     \vspace{0.1cm}
     
    \FOR{client $i=1,\ldots,N$}
    \IF{$ t\mod T E_i = 0$} 
    \NEWCOMMENT{Client $i$ has enough energy to participate in training.}
    \STATE Sample an integer $J$ uniformly random from $\{0, \ldots, E_i-1\}$.    
    \STATE Update $I_i^{t+JT+l} \leftarrow 1$ for $l\in\{0, \ldots, T-1\}$. 
    \NEWCOMMENT{Client $i$ is scheduled at global round $t+JT$.}
    \ENDIF

    \vspace{0.1cm}
    \IF{$t \mod T = 0$} 
    \IF{$I_i^t = 1$}
    \NEWCOMMENT{Client $i$ locally updates the model.}
    \STATE Initialize the local model $\mathbf{w}_i^{(t)} \leftarrow \mathbf{w}^{(t)}$
     \FOR{iteration $j=0,\ldots,T-1$}
    \STATE Update the local model according to \eqref{localupdate}. 
     \ENDFOR
    \STATE Send the local update $\mathbf{g}_i^{(t)} $ from \eqref{localupdateg}   to the server. 
    \ENDIF
    \ENDIF
     \ENDFOR
  
\vspace{0.1cm}
    \NEWSECTION{Server:}
    \vspace{0.1cm}
    
    \IF{$ (t+1)\mod T = 0$} 
    \STATE Receive the local updates $\mathbf{g}_i^{(t+1)}$ from the clients in  $\mathcal{S}_t = \{i: I_i^t = 1\}$. 
    \STATE Update the global model according to  \eqref{globalmodelupdate},
    \begin{equation}
\mathbf{w}^{(t+1)} = \mathbf{w}^{(t-T+1)} +  \sum_{i\in S_t} p_i \mathbf{g}_i^{(t+1)} 
\end{equation}
    \STATE  Send the updated global model $\mathbf{w}^{(t+1)}$ to the clients. 
     \ENDIF

    \ENDFOR

  \end{algorithmic}
\end{algorithm}

\subsection{Client scheduling} \label{scheduling}
The first component of our framework is client scheduling for training. Client selection in conventional federated learning algorithms are primarily based on the assumption that all clients are inherently available to participate in training if chosen, or that client dropouts occur uniformly at random (which does not bias the training), and focus on selecting the clients  to maximize the  convergence rate or to reduce the communication overhead of training \cite{mcmahan2017communication, bonawitz2017practical, LiHYWZ20, chen2020optimal, cho2020client}. In contrast, in our setup, not all clients can participate in the training process at all rounds. In particular, if a client has participated in one round, they may not have enough energy to participate in the next round. Moreover, the energy availability of the clients is non-uniform, i.e., some clients have less frequent energy arrivals than others. 

A naive approach for client scheduling is to schedule clients as soon as they have collected enough energy  to participate in training. 
However, as we will demonstrate in our experiments, doing so can bias the trained model towards clients with better (more frequent) energy availability.   
Another approach is to wait until all clients become available for training, and then use a conventional client sampling algorithm. However, waiting for all clients to have enough energy can significantly increase the total training time needed to achieve a target performance level. 

Instead, we propose a simple client scheduling protocol that can be performed locally by the clients. In our protocol, clients participate in training through a stochastic process based on their energy profile. The details of this process is as follows. 
First, we note that it takes $E_i$ global rounds for client $i$ to harvest enough energy to participate in one global round of training\footnote{For simplicity, we assume that  when $t=0$, all clients have enough energy to participate in one global round. Our results hold even if clients start at different time instances.}. When $t\mod (E_i T) = 0$, client $i$ samples an integer $J$ uniformly at random from the set $\{0, \ldots, E_i-1\}$. Then, within the $E_i$ global rounds starting at the time instances $\{t, t+T, \ldots, t+(E_i-1)T\}$, client $i$ only participates during the global round that starts at $t+JT$, and does not participate in the remaining global rounds.

We note that the client selection algorithm decides whether or not a client will participate at a given global round. 
If a client chooses to participate at a global round starting at some $t\in\mathcal{T}$, then the client participates for the whole duration of that global round, i.e., for $\{t, \ldots, t+T-1\}$, by computing the local model as in  \eqref{localupdate}.
As such, for any global round starting at $t\in\mathcal{T}$, the set of participating clients $\mathcal{S}_t$ at time $t, \ldots, t+T-1$ satisfy,
\begin{equation}
\mathcal{S}_t = \mathcal{S}_{t+1} = \ldots =  \mathcal{S}_{t+T-1},
\end{equation}
i.e., the set of clients participating in a given global round stays the same throughout the duration of that global round.


As we demonstrate in our theoretical analysis, the proposed client scheduling strategy provides provable convergence guarantees for the global model. 


\subsection{Local training}\label{sec:localtrain}


At the beginning of each global round, the server sends the current state of the global model to the clients. 
The clients then locally update the global model on their local datasets. 

The details of this local update process is as follows. 
Consider a global round starting at some $t\in\mathcal{T}$. Then, the server sends the current state of the global model $\mathbf{w}^{(t)}$ to the clients. 
Then, clients choose whether or not to participate in the current global round, based on the client scheduling process from Section~\ref{scheduling}. Clients who choose to participate in the current global round then compute a local model, by updating the global model $\mathbf{w}^{(t)}$ through $T$ local SGD iterations as in \eqref{localupdate}. 

After $T$ local SGD iterations, client $i\in\mathcal{S}_{t+T-1}$ sends a  local update to the server. The local update is defined as,
\begin{equation}\label{localupdateg}
\mathbf{g}_i^{(t+T)} \triangleq E_i (\mathbf{w}_i^{(t+T)} - \mathbf{w}^{(t)})
\end{equation}
which is obtained by shifting $\mathbf{w}_i^{(t+E)}$ by $\mathbf{w}^{(t)}$ and then scaling it  with respect to the energy renewal cycle $E_i$.  



\subsection{Global model update}
After receiving the local updates in \eqref{localupdateg} from the participating clients, the server updates the global model as,
\begin{equation}\label{globalmodelupdate}
\mathbf{w}^{(t+T)} = \mathbf{w}^{(t)} +  \sum_{i\in S_{t+T-1}} p_i \mathbf{g}_i^{(t+T)} 
\end{equation}
and sends the updated model $\mathbf{w}^{(t+T)}$ back to the clients, for the next iteration. 

We note that the complexity of Algorithm~\ref{Alg} is the same as that of  conventional federated learning, i.e., the federated averaging algorithm (\emph{FedAvg}) from \cite{mcmahan2017communication}. 
In the following, we demonstrate the theoretical convergence guarantees of our framework.

\section{Convergence Analysis} 

In this section, we provide the convergence guarantees of our framework. First, we review a few common technical assumptions  \cite{LiHYWZ20, stich2018local} that will be useful in our further analysis. 




\begin{assumption}(Strong-convexity) The local loss functions $F_i(\mathbf{w})$ for $i\in[N]$ are $\mu$-strongly convex:
\begin{equation}\label{eq:strong}
F_i(\mathbf{v}) \geq  F_i(\mathbf{w}) + F_i(\mathbf{w})^{T}(\mathbf{v} - \mathbf{w})  + \frac{\mu}{2} ||\mathbf{w} - \mathbf{v}||^2
\end{equation}
\end{assumption}

\begin{assumption}(Smoothness) The local loss functions $F_i(\mathbf{w})$ for $i\in[N]$ are $L$-smooth:
\begin{equation}\label{eq:smooth}
F_i(\mathbf{v}) \leq  F_i(\mathbf{w}) + F_i(\mathbf{w})^{T}(\mathbf{v} - \mathbf{w})  + \frac{L}{2} ||\mathbf{w} - \mathbf{v}||^2
\end{equation}
\end{assumption}

\begin{assumption}(Variance bound) The stochastic gradient $\nabla F_i(\mathbf{w}_i^{(t+j)}, \xi_{i}^{(t+j)})$ has bounded variance for all $i\in [N]$
\begin{equation}\label{eq:bounded-var}
E[||\nabla F_i(\mathbf{w}_i^{(t+j)}, \xi_{i}^{(t+j)}) -  \nabla F_i(\mathbf{w}^{(t)})||^2]\leq \sigma^2. 
\end{equation}
\end{assumption}

\begin{assumption}(Bounded second moment) The stochastic gradient $\nabla F_i(\mathbf{w}_i^{(t+j)}, \xi_{i}^{(t+j)})$ has bounded expected squared norm for all $i\in [N]$,
\begin{equation}\label{eq:bounded-norm}
E[||\nabla F_i(\mathbf{w}_i^{(t+j)}, \xi_{i}^{(t+j)})||^2]\leq G^2. 
\end{equation}
\end{assumption}


Our convergence analysis is simple and follows along the lines of standard convergence analysis techniques for distributed SGD with local averaging \cite{LiHYWZ20, stich2018local, cho2020client}. 

We first represent the model update process in an equivalent but more tractable form. Note that in our original problem formulation, at any global round, only the clients that choose to participate in that global round perform the local update. To make the mathematical analysis simpler, one can instead assume that all clients perform local training at each global round, but the global model is updated by using only the local updates from the clients that were originally scheduled at that global round. 
Note that, mathematically, the two processes lead to the same global model. 
Hence, we will use the latter approach in the following, and represent the model update process from Algorithm~\ref{Alg} as,
\begin{align}
&\mathbf{v}_k^{(t+1)}  = \mathbf{w}_k^{(t)} -  \eta_t \nabla F_k (\mathbf{w}_k^{(t)}, \xi_k^{(t)}) \label{vvkt}\\
&\mathbf{w}_k^{(t+1)}  = 
  \left \{\begin{matrix} \mathbf{v}_k^{(t+1)} &  \!t\!+\!1\!\notin  \mathcal{T} \\ \mathbf{w}_k^{(t+1-T)}  \!+\! \sum_{i\in\mathcal{S}_{t}} p_i E_i (\mathbf{v}_i^{(t+1)} \!-\! \mathbf{w}_i^{(t+1-T)}) &  \!t\!+\!1\!\in \mathcal{T}\end{matrix} \right . \label{vvkt2}
\end{align}
for $t\in[K]$ and $k\in[N]$. 
Note that for $t+1\in\mathcal{T}$, 
\begin{equation}\label{ekt1T}
\mathbf{w}_k^{(t+1-T)} = \mathbf{w}^{(t+1-T)}
\end{equation}
for all $k\in[N]$, since $t+1-T\in \mathcal{T}$ whenever $t+1\in\mathcal{T}$.

We next define two virtual sequences that will be useful in our analysis:
\begin{align}\label{eq:virtualw}
\overline{\mathbf{w}}^{(t+1)} = \sum_{k\in[N]} p_k \mathbf{w}^{(t+1)}_k 
\end{align}
and
\begin{align}\label{eq:virtualv}
\overline{\mathbf{v}}^{(t+1)} = \sum_{k\in[N]} p_k \mathbf{v}^{(t+1)}_k.  
\end{align}

We next provide a key lemma.  





\begin{lemma}\label{lemma1} (Unbiased client scheduling)
For all $t\in \mathcal{T}$, 
\begin{equation}\label{eq:unbiased}
\mathbb{E}\left[\overline{\bw}^{(t+1)} \right ]  = \overline{\mathbf{v}}^{(t+1)},
\end{equation}
hence, the client scheduling process from Section~\ref{scheduling} is unbiased.
\end{lemma}

\begin{proof}\label{proof-lemma1} 
Define a binary random variable $\alpha_i^{t}$ such that:
\begin{equation}\label{alpha}
\alpha_i^{t} = \left \{ \begin{matrix} 1 & \text{ if client $i$ participates at iteration $t$}\\ 0 & \text{ otherwise}\end{matrix} \right .
\end{equation}
According to the client scheduling algorithm from Section~\ref{scheduling}, at $t\in\mathcal{T}$, client $i$ participates in one of the $E_i$ consecutive global rounds uniformly at random. Hence, among the global rounds starting at $\{t, t+T, \ldots, t+(E_i-1)T\}$, the probability of participating at a specific round is $\frac{1}{E_i}$, from which we have,
\begin{equation}\label{eq:prob}
P[\alpha_i^{t} = 1]  = \frac{1}{E_i}. 
\end{equation}
Then, by defining $\alpha_t \triangleq (\alpha_1^t,\ldots, \alpha_N^t)$, we find that,
\begin{align}
 \mathbb{E}\left[\overline{\bw}^{(t+1)} \right ]  & = \sum_{k\in[N]} p_k  \mathbb{E}\Big [ 
\mathbf{w}^{(t+1-T)}_k   + \sum_{i\in\mathcal{S}_{(t+1)}} E_i p_i (\mathbf{v}^{(t+1)}_i - \mathbf{w}^{(t+1-T)}_i)  \Big ] \\
& = \sum_{k\in[N]} p_k \bigg ( \mathbf{w}^{(t+1-T)}_k  + \mathbb{E}\Big [ 
\sum_{i\in\mathcal{S}_{(t+1)}} E_i p_i (\mathbf{v}^{(t+1)}_i - \mathbf{w}^{(t+1-T)}_i)  \Big ]\bigg ) \\
& = \sum_{k\in[N]} p_k \bigg ( \mathbf{w}^{(t+1-T)}_k  + \mathbb{E}\Big [ 
\sum_{i\in[N]} \alpha_i^t E_i p_i (\mathbf{v}^{(t+1)}_i - \mathbf{w}^{(t+1-T)}_i)  \Big ]\bigg ) \\
& = \sum_{k\in[N]} p_k \bigg ( \mathbf{w}^{(t+1-T)}_k   + \sum_{i\in[N]} \mathbb{E} [ 
 \alpha_i^t E_i p_i (\mathbf{v}^{(t+1)}_i - \mathbf{w}^{(t+1-T)}_i)  ]\bigg ) \\
 & = \sum_{k\in[N]} p_k \bigg ( \mathbf{w}^{(t+1-T)}_k   +
\sum_{i\in[N]} \frac{1}{E_i} E_i p_i (\mathbf{v}^{(t+1)}_i - \mathbf{w}^{(t+1-T)}_i) \bigg ) \label{eq:322a} \\
 & = \mathbf{w}^{(t+1-T)} + \sum_{k\in[N]} p_k \bigg ( 
\Big ( \sum_{i\in[N]} p_i \mathbf{v}^{(t+1)}_i \Big ) - \mathbf{w}^{(t+1-T)} \bigg ) \label{eq:322b} \\
 & = \mathbf{w}^{(t+1-T)} +  \sum_{i\in[N]} p_i \mathbf{v}^{(t+1)}_i - \mathbf{w}^{(t+1-T)}  \label{eq:322c} \\
 & =  \sum_{i\in[N]} p_i \mathbf{v}^{(t+1)}_i  = \overline{\mathbf{v}}^{(t+1)} \label{eq:322d}
\end{align}
where \eqref{eq:322a} is from \eqref{eq:prob}, and \eqref{eq:322b} is from \eqref{ekt1T} and \eqref{eq:sump}. 
\end{proof}
Next, we provide another key lemma. 
\begin{lemma}\label{lemma2} (Bounded variance for $\overline{\bw}^{(t+1)}$)
For all $t\in \mathcal{T}$, by assuming a decreasing learning rate $\eta_t$ with $\eta_t \leq 2\eta_{t+T}$ for all $t\geq0$, we have 
\begin{equation}\label{eq:unbiased2}
\mathbb{E} [\lVert\overline{\mathbf{v}}^{(t+1)}  -  \overline{\bw}^{(t+1)}   \rVert^2 ]  \leq 4 E_{max}^2 G^2 \eta_t^2 T^2,
\end{equation}
where $E_{max} = \max_{i\in[N]} E_i$.  
Hence, the aggregate of the local models have bounded variance. 
\end{lemma}
\begin{proof}
From \eqref{eq:virtualw} and \eqref{eq:virtualv}, we have that,
\begin{align}
& \mathbb{E} [\lVert\overline{\mathbf{v}}^{(t+1)}  -  \overline{\bw}^{(t+1)}   \rVert^2 ] \notag \\
& \quad =
\mathbb{E} [\lVert \underbrace{\sum_{k\in[N]} p_k\mathbf{v}^{(t+1)}_k}_{A}   -  \underbrace{\sum_{k\in[N]} p_k\mathbf{w}^{(t+1)}_k}_{B}   \rVert^2 ]  \label{eq:AB}
\end{align}
The first term in \eqref{eq:AB} can be written as:
\begin{align}
A & = \sum_{k\in[N]} p_k \Big ( \bw^{(t+1-T)}_k \!- \! \sum_{j=t+1-T}^{t} \eta_j \nabla F_k (\mathbf{w}_k^{(j)}, \xi_k^{(j)})   \Big ) \label{wkt}\\
& = \bw^{(t+1-T)}  - \sum_{k\in[N]}  \sum_{j=t+1-T}^{t}  p_k \eta_j \nabla F_k (\mathbf{w}_k^{(j)}, \xi_k^{(j)})  
\label{eq:222}
\end{align}
where \eqref{wkt} follows from \eqref{vvkt}, and \eqref{eq:222} is 
from \eqref{eq:sump} and \eqref{ekt1T}. 
Similarly, the second term in  \eqref{eq:AB} can be written as:
\begin{align}
B & = \sum_{k\in[N]} p_k \Big ( \mathbf{w}_k^{(t+1-T)} \!+\! \sum_{i\in\mathcal{S}_t} E_i p_i (\mathbf{v}_i^{(t+1)} \!-\! \mathbf{w}_i^{(t+1-T)})  \Big ) \label{eq:B1} \\
& = \mathbf{w}^{(t+1-T)} +  \sum_{i\in\mathcal{S}_t} E_i p_i (\mathbf{w}^{(t+1-T)} - \sum_{j=t+1-T}^{t} \eta_j \nabla F_i(\mathbf{w}_i^{(j)}, \xi_i^{(j)}) - \mathbf{w}^{(t+1-T)}) \label{eqB2}\\
& = \mathbf{w}^{(t+1-T)} -  \sum_{i\in\mathcal{S}_t} E_i p_i \sum_{j=t+1-T}^{t} \eta_j \nabla F_i(\mathbf{w}_i^{(j)}, \xi_i^{(j)}) \label{eqB3}
\end{align}
where \eqref{eq:B1} follows from \eqref{vvkt2}, \eqref{ekt1T}, and \eqref{eq:sump}, whereas  \eqref{eqB2} is from \eqref{vvkt}. By combining \eqref{eqB3}, \eqref{eq:222}, and  \eqref{eq:AB}, we have,
\begin{align}
 \mathbb{E} [\lVert \sum_{k\in[N]} p_k\mathbf{v}^{(t+1)}_k   -  \sum_{k\in[N]} p_k\mathbf{w}^{(t+1)}_k   \rVert^2 ] 
& = \mathbb{E} [\lVert  
- \sum_{k\in[N]}  \sum_{j=t+1-T}^{t} p_k \eta_j \nabla F_k (\mathbf{w}_k^{(j)}, \xi_k^{(j)})  \notag \\
& \quad \quad +  \sum_{i\in\mathcal{S}_t} E_i p_i \sum_{j=t+1-T}^{t} \eta_j \nabla F_i(\mathbf{w}_i^{(j)}, \xi_i^{(j)})
\rVert^2 ] \\
& \leq \mathbb{E} [\lVert  
\sum_{i\in\mathcal{S}_t} E_i p_i \sum_{j=t+1-T}^{t} \eta_j \nabla F_i(\mathbf{w}_i^{(j)}, \xi_i^{(j)})
\rVert^2 ] \label{Ex2} 
\end{align}

where \eqref{Ex2} holds from  $E[(X-E[X])^2]\leq E[X^2]$ and that,
\begin{align}
 \mathbb{E} [ 
\sum_{i\in\mathcal{S}_t} E_i p_i \sum_{j=t+1-T}^{t} \eta_j \nabla F_i(\mathbf{w}_i^{(j)}, \xi_i^{(j)}) ] 
& = \mathbb{E} [\sum_{i\in[N]} \alpha_i^t E_i p_i \sum_{j=t+1-T}^{t} \eta_j \nabla F_i(\mathbf{w}_i^{(j)}, \xi_i^{(j)}) ] \\
& = \sum_{i\in[N]} \frac{1}{E_i} E_i p_i \sum_{j=t+1-T}^{t} \eta_j \nabla F_i(\mathbf{w}_i^{(j)}, \xi_i^{(j)})  \\
& = \sum_{k\in[N]}  \sum_{j=t+1-T}^{t} p_k \eta_j \nabla F_k (\mathbf{w}_k^{(j)}, \xi_k^{(j)}) 
\end{align}
by defining $\alpha_t \triangleq (\alpha_1^t,\ldots, \alpha_N^t)$ as in  \eqref{alpha}. 

Finally, by letting $\mathbf{r}_i^{j} \triangleq \nabla F_i(\mathbf{w}_i^{(j)}, \xi_{i}^{(j)})$, 
\begin{align}
\mathbb{E} [\lVert  
\sum_{i\in\mathcal{S}_t} E_i p_i \sum_{j=t+1-T}^{t} \eta_j \nabla F_i(\mathbf{w}_i^{(j)}, \xi_i^{(j)})
\rVert^2 ] 
&  =  
 \mathbb{E} [ 
\sum_{i\in\mathcal{S}_t} \sum_{{i'}\in\mathcal{S}_t}
E_i E_{i'}  p_i p_{i'}\inner{\sum_{j=t+1-T}^{t} \!\!\! \eta_j  \mathbf{r}_i^{j}}{ \!\!\!\sum_{j'=t+1-T}^{t} \!\!\!\eta_{j'} \mathbf{r}_{i'}^{j'} }
]  \\
&  \leq  \mathbb{E} [ 
\sum_{i\in\mathcal{S}_t} \sum_{{i'}\in\mathcal{S}_t}
E_i E_{i'}  p_i p_{i'} \sum_{j=t+1-T}^{t} \sum_{j'=t+1-T}^{t} \eta_j \eta_{j'} G^2  ]  \label{eq:lasteq} \\
&  \leq  \mathbb{E}[ 
\sum_{i\in [N]} \sum_{{i'}\in [N]} \hspace{-0.2cm}\alpha_i \alpha_{i'}
E_i E_{i'}  p_i p_{i'} \hspace{-0.2cm}\sum_{j=t+1-T}^{t} \sum_{j'=t+1-T}^{t} \hspace{-0.2cm}\eta_j \eta_{j'} G^2  ]  \notag  \\
&  \leq  
\sum_{i\in [N]} \sum_{{i'}\in [N]} \mathbb{E} [ \alpha_i \alpha_{i'}]
E_i E_{i'}  p_i p_{i'} \hspace{-0.2cm}\sum_{j=t+1-T}^{t} \sum_{j'=t+1-T}^{t} \hspace{-0.3cm} \eta_j \eta_{j'} G^2   \notag  \\
&  \leq  E_{max}^2  T^2 \eta_{t+1-T}^2 G^2 \leq 4 E_{max}^2  T^2 \eta_{t}^2 G^2  \label{eq:Emax}
\end{align}
where \eqref{eq:lasteq} follows from \eqref{eq:bounded-norm} and that,
\begin{align}
 \inner{\sum_{j=t+1-T}^{t} \eta_j  \mathbf{r}_i^{j}}{ \sum_{j'=t+1-T}^{t} \eta_{j'} \mathbf{r}_{i'}^{j'} } 
& = \sum_{j=t+1-T}^{t} \sum_{j'=t+1-T}^{t} \eta_{j} \eta_{j'}  \inner{\mathbf{r}_i^{j}}{\mathbf{r}_{i'}^{j'}} \\
& \leq \sum_{j=t+1-T}^{t} \sum_{j'=t+1-T}^{t} \eta_{j} \eta_{j'}   \lVert \mathbf{r}_{i}^{j} \rVert \lVert \mathbf{r}_{i'}^{j'} \rVert  \label{eq:CS} \\
& \leq  \sum_{j=t+1-T}^{t} \sum_{j'=t+1-T}^{t} \eta_{j} \eta_{j'}  
\frac{1}{2} (\lVert \mathbf{r}_{i}^{j} \rVert^2 + \lVert \mathbf{r}_{i'}^{j'} \rVert^2) \label{eq:AMGM}
\end{align}
where \eqref{eq:CS} is from the Cauchy-Schwarz inequality, and \eqref{eq:AMGM} is from the AM-GM (arithmetic mean-geometric mean) inequality. Equation \eqref{eq:Emax} follows from using a decreasing  learning rate $\eta_t$ with $t$ and $\eta_t \leq 2\eta_{t+T}$. In equation \eqref{eq:Emax}, we define $E_{max} \triangleq \max_{i\in [N]} E_i$. 
\end{proof}



We next define the degree of heterogeneity between the clients as in \cite{LiHYWZ20},
\begin{equation}
\Gamma  = F^* - \sum_{i\in[N]} p_i F_i^* \label{gamma}
\end{equation}
where $F^*$ and $F_i^*$ denote the minimum of the global and local loss functions from \eqref{eq:global1} and \eqref{localoss}, respectively.


We are now ready to state our convergence guarantees. 

\begin{theorem}\label{thm1-deter}
For the federated learning problem from \eqref{eq:global1} over $N$ clients and an energy renewal cycle $E_i$ for client $i\in[N]$, Algorithm~\ref{Alg} converges, 
\begin{align}\label{eq:conv}
&\mathbb{E} [F(\mathbf{w}^{(T)})] - F(\mathbf{w}^*) \notag \\
&\hspace{0.5cm} \leq  \frac{2\kappa}{\gamma + K } \left(  \frac{B+C}{\mu} + 2L \lVert \mathbf{w}^{(0)} - \mathbf{w}^*\rVert^2\right )
\end{align}
in $K$ iterations, where $\mathbf{w}^{*}$ denotes the optimal model parameters that minimize the global loss function in \eqref{eq:global1}, and
\begin{equation}\label{eq:C}
C \triangleq 4 E_{max}^2  T^2 \eta_{t}^2 G^2,  
\end{equation} 
where $E_{max} \triangleq \max_{i\in [N]} E_i$, $\kappa = \frac{L}{\mu}$, $\gamma = \max\{8\kappa, T\}$, learning rate $\eta_t = \frac{2}{\mu (\gamma + t)}$, and $B = \sigma^2 6L\Gamma + 8(T-1)^2 G^2$. 
\end{theorem}
\begin{proof} 
The proof follows directly from Lemmas~\ref{lemma1} and \ref{lemma2} along with standard steps in the convergence analysis of distributed SGD with local averaging \cite{ LiHYWZ20, cho2020client, stich2018local}, e.g., from Section B.3 of \cite{LiHYWZ20} by replacing Lemmas~4 and 5 from  \cite{LiHYWZ20} with Lemmas~\ref{lemma1} and ~\ref{lemma2} from our work, respectively. 
\end{proof}


\section{Experiments}


We now demonstrate the convergence properties of Algorithm~\ref{Alg} compared to  federated learning benchmarks that are agnostic to the energy availability of the clients. 

{\bf Network architecture.} We consider an image classification task with $10$ classes on the CIFAR-10 dataset \cite{krizhevsky2009learning}. 
Training is done using a convolutional neural network with the same architecture from  \cite{mcmahan2017communication} (about $10^6$ parameters). 

\begin{figure}[t]
\centering
\includegraphics[width=0.6\linewidth]{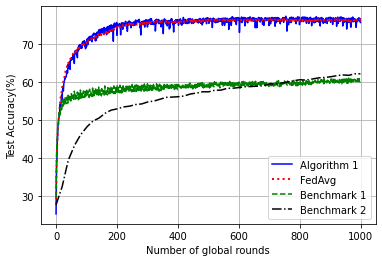}
\caption{Test accuracy of Algorithm~\ref{Alg} compared to federated learning benchmarks  for $N=40$ clients on the CIFAR-10 dataset. }
\label{fig:test-accuracy}
\end{figure}



{\bf Experiment Setup.} We consider a network of $N=40$ clients. The dataset is distributed in an i.i.d. fashion across the network, by shuffling the dataset and distributing it evenly across the clients. Clients use the ADAM optimizer \cite{kingma2014adam} during training, and the number of local training iterations is set to $T=5$.   


{\bf Energy Profile.} In order to show the impact of non-homogeneous energy arrivals, clients are partitioned into $4$ groups $\mathcal{U}_0, \ldots, \mathcal{U}_3$ of equal size, such that  $\mathcal{U}_ k =\{i: i\mod 4 = k\}$. Then, the energy arrivals of clients in group $\mathcal{U}_k$ are assigned as $E_i = \tau_k$ for all $i\in\mathcal{U}$, where $(\tau_0, \tau_1, \tau_2, \tau_3) = (1, 5, 10, 20)$. 
In other words, clients in group $\mathcal{U}_0$ receive energy at every global round, whereas clients in groups  $\mathcal{U}_1$, $\mathcal{U}_2$, and $\mathcal{U}_3$ receive energy at every $5$, $10$, and $20$ global rounds, respectively. 

{\bf Benchmarks.} 
To evaluate the performance of Algorithm~\ref{Alg}, we consider the conventional federated learning algorithm from Section~\ref{federated} (known as \emph{FedAvg} \cite{mcmahan2017communication}), but under the constraint that users receive energy according to the energy arrival process in Section~\ref{sec:energy}, and implement two benchmarks with respect to the specific client scheduling policy. 



\noindent
{\it Benchmark 1:} 
In the first benchmark, each client participates in training as soon as they have enough energy, and then waits until the next energy arrival. More specifically, as soon as $t\mod TE_i = 0$, the client participates in the current global round,  by updating the current state of the global model through $T$ local training iterations as in  \eqref{localupdate}, and then sending the local update to the server. The server then updates the global model according to  \eqref{eq:aggregation}. The client does not participate in training in the next $E_i-1$ global rounds, until the next energy arrival.  

\noindent
{\it Benchmark 2:} 
In the second benchmark, the global model is updated only when all clients have received energy, i.e., the  server waits until all clients have energy available before initiating a global update. 
After all clients have received energy, the server sends the current state of the global model to the clients, the clients compute a local model as in \eqref{localupdate}, and then the server aggregates the local models to update the global model as in \eqref{eq:aggregation}. 
Note that in this case, the server needs to wait for the slowest client, hence the global model is updated once in the duration of $20$ global rounds. 


We evaluate the training performance in terms of the test accuracy with respect to the number of global rounds. 
Our results are given in Figure~\ref{fig:test-accuracy}. In our experiments, we also implement the original federated learning algorithm (\emph{FedAvg}) from Section~\ref{federated} without any resource limitations, which acts as an upper bound on the accuracy.  We observe that Algorithm~\ref{Alg} achieves an accuracy of $77\%$, which is comparable to the accuracy of {\it FedAvg}, whereas the accuracy of the two benchmarks are $60\%$ and $62\%$, respectively, within $1000$ global rounds. 
This is caused by the fact that, in the first benchmark, the training algorithm favors clients with more frequent energy availability, which causes the global model to be biased. In the second benchmark, the server waits until all clients have energy available before each global update, which causes the convergence to be very slow even though the algorithm is unbiased. On the other hand, Algorithm~\ref{Alg} converges fast and significantly outperforms the benchmarks in terms of test accuracy.  


\vspace{-0.2cm}\section{Conclusion}\label{conclusion} 
This paper proposes sustainable federated learning with the  utilization of intermittently powered devices, where a large number of remote devices are expected to perform training on a daily basis. 
We demonstrate a simple and scalable federated learning strategy with provable convergence guarantees, for devices with intermittent energy availability, and show that the proposed framework can significantly improve the training performance compared to the energy-agnostic benchmarks.  
We hope our work to open up further research on  sustainable learning in large-scale federated and decentralized settings.

\bibliographystyle{IEEEtran}
\bibliography{ref}

\end{document}